\documentclass[preprint,3p]{elsarticle}
\usepackage{amssymb}
\usepackage{amsfonts}
\usepackage{amsmath}
\usepackage{amsthm}
\usepackage{enumerate}
\usepackage{algorithmic}
\usepackage{algorithm}
\usepackage{multirow}
\usepackage{graphicx,color}
\usepackage{xcolor}
\usepackage{IEEEtrantools}
\usepackage{booktabs}
\usepackage{subfigure}
\usepackage{natbib} \biboptions{authoryear}
\usepackage[colorlinks, citecolor=blue, linkcolor=black]{hyperref}
\usepackage{times}
\usepackage[OT1]{fontenc}
\newcommand{\bPsi}{\boldsymbol{\Psi}}\newcommand{\tbPsi}{\tilde{\bPsi}}\newcommand{\bS}{\mathbf{S}}\newcommand{\bB}{\mathbf{B}}
\newcommand{\bX}{\mathbf{X}}
\newcommand{\bT}{\mathbf{T}}\newcommand{\bR}{\mathbf{R}}\newcommand{\bU}{\mathbf{U}}
\newcommand{\bA}{\mathbf{A}}\newcommand{\tbA}{\tilde{\mathbf{A}}}

\newcommand{\ba}{\mathbf{a}}\newcommand{\tpsi}{\tilde{\psi}}\newcommand{\tomega}{\tilde{\omega}}
\newcommand{\tba}{\tilde{\ba}}
\newcommand{\bI}{\mathbf{I}}
\newcommand{\bbE}{\mathbb{E}}\newcommand{\bW}{\mathbf{W}}
\newcommand{\cC}{\mathcal{C}}\newcommand{\cD}{\mathcal{D}}\newcommand{\cF}{\mathcal{F}}\newcommand{\cL}{\mathcal{L}}\newcommand{\cP}{\mathcal{P}}\newcommand{\bLmd}{\boldsymbol{\Lambda}}
\newcommand{\cO}{\mathcal{O}}
\newcommand{\bZ}{\mathbf{Z}}\newcommand{\bSig}{\boldsymbol{\Sigma}}

\newcommand{\bl}{\mathbf{1}}\newcommand{\bo}{\mathbf{0}}\newcommand{\be}{\mathbf{e}}
\newcommand{\bx}{\mathbf{x}}\newcommand{\bz}{\mathbf{z}}

\newcommand{\bu}{\mathbf{u}}\newcommand{\bb}{\mathbf{b}}
\newcommand{\cN}{\mathcal{N}}

\newcommand{\beps}{\boldsymbol{\epsilon}}\newcommand{\bgamma}{\boldsymbol{\gamma}}\newcommand{\bmu}{\boldsymbol{\mu}}
\newcommand{\btheta}{\boldsymbol{\theta}}

\newcommand{\bvarphi}{\boldsymbol{\varphi}}
\newcommand{\tmu}{\tilde{\mu}}\newcommand{\tbmu}{\tilde{\bmu}}

\newcommand\refe[1]{(\ref{#1})}
\newcommand\reff[1]{Fig.~\ref{#1}}\newcommand\reft[1]{Tab.~\ref{#1}}\newcommand\refs[1]{Sec.~\ref{#1}}

 \newcommand\refthm[1]{Theorem~\ref{#1}}\newcommand\refobs[1]{Observation~\ref{#1}}
\newtheorem{thm}{Theorem}
\newtheorem{obs}{Observation}

\def\diag{\mbox{diag}}
\def\cov{\mbox{cov}}

\def\s2{\sigma^2}

\begin{document}
\begin{frontmatter}
\title{Choosing the number of factors in factor analysis with incomplete data via a hierarchical Bayesian information criterion}
 \author[Jianhua]{Jianhua Zhao} \ead{jhzhao.ynu@gmail.com}
 \cortext[cor1]{Corresponding author.}
\author[Jianhua,C. Shang]{Changchun Shang} 
\author[S. Li]{Shulan Li} 
 \author[L. Xin]{Ling Xin}
\author[Philip]{Philip L.H. Yu\corref{cor1}} \ead{plhyu@eduhk.hk} %
\address[Jianhua]{School of Statistics and Mathematics, Yunnan University of Finance and Economics, Kunming, 650221, China.} 
\address[C. Shang]{College of Science, Guilin University of Technology, Guilin, 541006, China}
\address[S. Li]{School of Accounting, Yunnan University of Finance and Economics, Kunming, 650221, China.}
 \address[L. Xin]{Division of Business and Management, BNU-HKBU United International College, Zhuhai, 519087, China.}
 \address[Philip]{Department of Mathematics and Information Technology, The Education University of Hong Kong}
 

\begin{abstract}
The Bayesian information criterion (BIC), defined as the observed data log likelihood minus a penalty term based on the sample size $N$, is a popular model selection criterion for factor analysis with complete data. This definition has also been suggested for incomplete data. However, the penalty term based on the `complete' sample size $N$ is the same no matter whether in a complete or incomplete data case. For incomplete data, there are often only $N_i<N$ observations for variable $i$, which means that using the `complete' sample size $N$ implausibly ignores the amounts of missing information inherent in incomplete data. Given this observation, a novel criterion called hierarchical BIC (HBIC) for factor analysis with incomplete data is proposed. The novelty is that it only uses the actual amounts of observed information, namely $N_i$'s, in the penalty term. Theoretically, it is shown that HBIC is a large sample approximation of variational Bayesian (VB) lower bound, and BIC is a further approximation of HBIC, which means that HBIC shares the theoretical consistency of BIC. Experiments on synthetic and real data sets are conducted to access the finite sample performance of HBIC, BIC, and related criteria with various missing rates. The results show that HBIC and BIC perform similarly when the missing rate is small, but HBIC is more accurate when the missing rate is not small.

\end{abstract}
\end{frontmatter}
\section{Introduction}\label{sec:intro}
Factor analysis (FA), which aims to identify the common characteristics among a set of variables, is a useful tool for data visualization, interpretation, and analysis. The parameter estimation can be easily performed using maximum likelihood (ML) method via the popular expectation maximization (EM)-like algorithm \citep{rubin_mlfa,chuanhai_ecme,zhao2008-efa}. In addition, the covariance structure of FA offers significant advantages over full/diagonal/ scalar covariance in density modeling for high-dimensional data, because of its capability of providing an appropriate trade-off between overfitting
full covariance and underfitting diagonal/scalar covariance \citep{bishop-ppca}.

For complete data, many model selection criteria can be adopted to find the trade-off, namely to determine the number of factors $q$, e.g., Akaike's information criterion (AIC) \citep{Akaike1987}, consistent AIC (CAIC) \citep{bozdogan1987model}, Bayesian information criterion (BIC) \citep{Schwarz1978-BIC}, etc. Among them, BIC is one very popular criterion, due to its theoretical consistency \citep{Shao1997-BIC} and satisfactory performance in applications. Formally, BIC is defined as the observed data log likelihood minus a penalty term depending on the sample size $N$. For data with missing values, \cite{Song2008-misfa} have suggested that the BIC with the `complete' sample size $N$ in the penalty term still can be used to determine the number of factors. However, to our knowledge, it seems that a theoretical justification for why BIC can be used for incomplete data is still missing. 

More importantly, for incomplete data, the actual sample size is only $N_i$ $(<N)$ at the level of variable $i$, which means that the penalty term of BIC, using the `complete' sample size $N$, implausibly ignores the amounts of missing information inherent in incomplete data, namely $(N-N_i)$'s. In \refs{sec:mot}, we consider a simple $d$-parameter model with a set of $N$ incomplete realizations of multivariate vector $\bx$. Under this model, the BIC approximation can be applied at two hierarchical levels. At the traditional higher level of vector $\bx$, the first BIC penalty is $(d/2)\log{N}$, but at the lower level of variables $x_i$, the second BIC penalty is $\sum_{i=1}^d(1/2)\log{N_i}$, which is lighter than the first one. The second penalty is appealing since it uses the actual amounts of observed information, namely $N_i$'s.     

Inspired by the second penalty, we are interested in developing a better criterion than BIC for factor analysis with incomplete data that makes use of the actual amounts of observed information, and investigating its performance in model selection. In this paper, we take the typical assumption that the missingness mechanism is missing at random (MAR) \citep{little_samd}. That is, given the observed part of $\bx$, the missingness does not depend on the missing part. We propose in this paper a novel criterion we call hierarchical BIC (HBIC). The novelty is that the approximation is performed at the lower level of the parameter $\btheta_i$ that specifies variables $x_i$, rather than the traditional higher level of $\btheta$ that specifies the whole vector $\bx$. Consequently, the proposed HBIC differs from BIC in that it penalizes the parameter $\btheta_i$ only using $N_i$, the actual sample size of variable $x_i$, rather than the `complete' sample size $N$. Theoretically, we show that (i) HBIC is a large sample approximation of variational Bayesian (VB) lower bound \citep{bishop-prml}; (ii) BIC is a further approximation of HBIC by dropping an order-1 term that does not depend on $N$. This means that HBIC shares the theoretical consistency of BIC. However, this order-1 term can be useful for incomplete data with finite sample size, as will be seen from our experiments in \refs{sec:expr}.

The remainder is organized as follows. In \refs{sec:mot}, we consider a simple model that inspires our problem. In \refs{sec:fa.bic}, we review FA model and the use of BIC under FA with complete and incomplete data. In \refs{sec:bic.obs}, we propose HBIC for incomplete data. We conduct an empirical study to compare HBIC and BIC in \refs{sec:expr}. We end the paper with some concluding remarks in \refs{sec:cons}.

\subsection{Notations} 
The following notations are used throughout this paper. Let $\bX_{obs}=\{\bx_n^o\}_{n=1}^N$ is a set of incomplete realizations of the $d$-dimensional $\bx=(x_1,x_2,\dots,x_d)'$ and the dimensionality of $\bx_n^o$ is $d_n$. $O_i$ denotes the set of indices $n$ for which $x_{ni}$ is observed and $N_i$ is the number of elements in $O_i$, namely the observed sample size or number of observed values of $x_i$. $O_n$ denotes the set of indices $i$ for which $x_{ni}$ is observed. Let $\bl_d$, $\bo_d$ stands for the $d$-dimensional vector whose all entries equal 1, 0, respectively. $\bI_d$ denotes a $d\times d$ identity matrix. For notation convenience, the subscript $d$ in $\bl_d$, $\bo_d$ and $\bI_d$ will be dropped if it is apparent from the context. $\mbox{blkdiag}(\bA,\bB)$ denotes the block diagonal concatenation of matrix $\bA$ and $\bB$. \mbox{linspace}$(a, b, d)$ denotes a vector of $d$ linearly
equally spaced points between $a$ and $b$, namely $(a,a+(b-a)/(n-1),\dots,b)'$.

\section{Motivation}\label{sec:mot}
We begin by constructing a simple model that captures the essential features of our problem. Suppose that a $d$-dimensional random vector $\bx$ follows the multivariate normal distribution $\cN(\bmu,\bI)$ with mean $\bmu=(\mu_1,\mu_2,\dots,\mu_d)'$ and the identity covariance matrix $\bI$. Obviously, the number of parameters in this model is $d$. Given $N$ incomplete realizations $\bX_{obs}=\{\bx_n^o\}_{n=1}^N$, and the sample size of variable $x_i$ being $N_i(< N)$. There are two ways to apply BIC approximation.
\begin{enumerate}[(i)]
\item Approximation 1: as shown in \reff{fig:demo}~(a), at the traditional higher level of vector $\bx$, the definition given by \cite{Song2008-misfa} (detailed in \refs{sec:bicobs}) yields the penalty $(d/2)\log{N}$, which is the same as that for complete data.
\item Approximation 2: as shown in \reff{fig:demo}~(b), at the lower level of variables $x_i$, since $x_i$'s are independent of each other, that is, $\log{p(\bx)}=\sum_{i=1}^d\log{p(x_i)}$, it is reasonable to apply BIC approximation to each log marginal distribution $\log{p(x_i)}$, which gives $(1/2)\log{N_i}$. Summing over all $d$ variables yields a new penalty $\sum_{i=1}^d(1/2)\log{N_i}$.
\end{enumerate}

Obviously, these two penalties are the same only when $N_i=N,i=1,\cdots d$, and can be significantly different when $N_i$ is much smaller than $N$. Intuitionaly, Approximation 2 would be more accurate since it uses the actual amounts of observed information, namely $N_i$'s. 

However, Approximation 2 is too limited because of the strong assumption that the variables are independent of each other. In this paper, we will consider the general case where there are correlations among variables. To be specific, we will consider factor analysis model and develop a new criterion applicable to such a general case. As will be seen in \refs{sec:bic.obs}, under this simple model, our proposed HBIC will degenerate to Approximation 2, which makes clear the significance of developing new criteria that only consider the actual amounts of observed information. 

\begin{figure*}[tbh]
	\centering
	\subfigure[]{\includegraphics[scale=0.8]{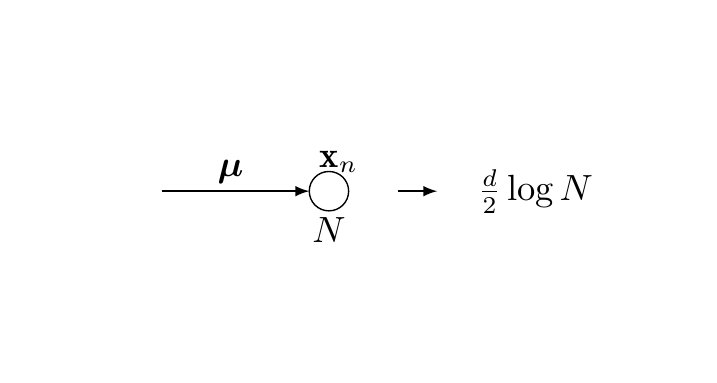}
	}
	\subfigure[]{\includegraphics[scale=0.8]{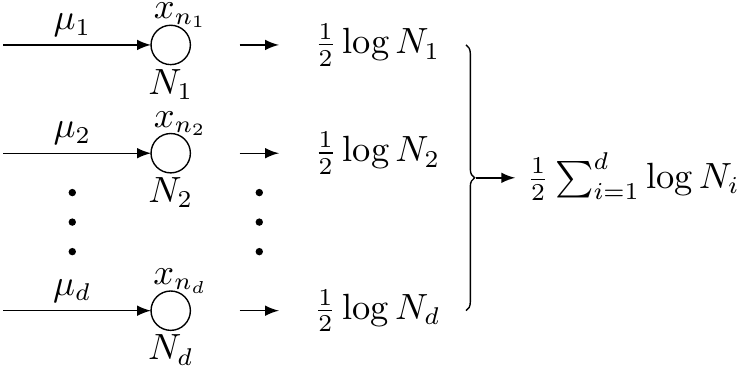}
	}		
	\caption{Two BIC approximations at two hierarchical levels. (a) Approximation 1: higher level of multivariate vector $\bx$; Approximation 2: lower level of variables $x_i$, $i=1,2,\dots,d$.}
	\label{fig:demo}
\end{figure*}
%
\section{Factor analysis and Bayesian information criterion}\label{sec:fa.bic}
\subsection{Factor analysis (FA) model}\label{sec:fa}
The classical $q$-factor model \citep{lawley_fasm} is defined as
\begin{equation}\label{eqn:fa}
x_i=\ba_i'\bz+\mu_i+\epsilon_i,\quad i=1,2,\dots,d,
\end{equation}
where $\bx=(x_1,x_2,\dots,x_d)'$ is a $d$-dimensional data vector, $\bmu=(\mu_1,\mu_2,\dots,\mu_d)'$ is a $d$-dimensional mean vector, $\bA=(\ba_1';\cdots;\ba_d')$ is a $d\times
k$ factor loading matrix with the column vector $\ba_i$ being the row $i$ of $\bA$, $\bz$ is a $k$-dimensional latent
factor vector which is assumed to follow the standard $k$-variate normal distribution $\cN(\mathbf{0},\bI)$, the error term $\beps=(\epsilon_1,\epsilon_2,\dots,\epsilon_d)'$ is a $d$-dimensional unique factor vector, which is assumed to follow normal $\cN(\bo,\bPsi)$ and independent of $\bz$. Here, $\bPsi=\hbox{diag}\{\psi_1,\psi_2,\dots,\psi_d\}$ is a positive diagonal matrix.

Under model \refe{eqn:fa}, $\bx\sim\cN(\bmu,\bSig)$,
where $\bSig=\bA\bA'+\bPsi$. It can be seen from \refe{eqn:fa} that FA model is invariant if we replace $\bA$ by $\bA\bR$ and $\bz$ by $\bR'\bz$, where $\bR$ is an orthogonal matrix, which means that the estimate of $\bA$ can only be determined up a rotation. Thus the number of free parameter in FA model is $\cD(k)=d(k+2)-k(k-1)/2$ \citep{lawley_fasm}. To avoid over-parameterization, the number of degrees of freedom in $\bSig$ should not exceed that of a full $d\times d$ covariance matrix, $d(d+1)/2$ \citep{Beal2003-phd}, which yields:
\begin{equation*}
{k}_{max}\leq d+\frac12(1-\sqrt{1+8d}).\label{eqn:q.dim}
\end{equation*}
Given a set of i.i.d
observations $\bX=\{\bx_n\}_{n=1}^N$, the data log likelihood is
\begin{equation}
\cL(\btheta)=-\frac12\sum\nolimits_{n=1}^N\{d\log{(2\pi)}+\log{|\bSig|}+(\bx_n-\bmu)'\bSig^{-1}(\bx_n-\bmu)\}.
\label{eqn:likelihood}
\end{equation}
The maximum likelihood estimate of $\btheta$ that maximizes $\cL$ in \refe{eqn:likelihood} can be easily found by many algorithms such as the expectation maximization (EM) \citep{rubin_mlfa}, the parameter-expanded EM (PX-EM) \citep{chuanhai_ecme}, conditional maximization (CM) algorithm \citep{zhao2008-efa}, etc. 


\subsection{Bayesian information criterion (BIC) for factor analysis with complete data}\label{sec:bicc}
Since the log likelihood $\cL$ under FA model is a
nondecreasing function of the number of factors
$k$, it can not be adopted as model complexity
criterion. Several model selection criteria have been proposed to deal with this problem. Three typical criteria are Akaike's information criterion (AIC) \citep{Akaike1987}, consistent AIC (CAIC) \citep{bozdogan1987model} and Bayesian information criterion (BIC) \citep{Schwarz1978-BIC}. They can be unified into the following form
\begin{equation}
	\cL^*(k,\hat{\btheta}(k))=\cL(\bX|\hat{\btheta}(k))-\frac{\cD(k)}2\cC(N),\label{eqn:cri}
\end{equation}
where $\hat{\btheta}(k)$ denotes the ML estimate of parameter $\btheta$ in $k$-factor model, $\cD(k)=d(k+2)-k(k-1)/2$ is the number of free parameters, and $(\cD(k)/2)\cC(N)$ is a penalty term that penalizes the higher values of $k$. In addition, $\cC(N)=2$ for AIC, $\cC(N)=\log{N}+1$ for CAIC, and $\cC(N)=\log{N}$ for BIC.

Among them, BIC is one very popular criterion for determining the number of factors in FA model, due to its theoretical consistency \citep{Schwarz1978-BIC} and satisfactory performance in applications. The penalty term of BIC can be written as
\begin{IEEEeqnarray}{rCl}
	\cP_{bic}(\hat{\btheta}(k))&=&\frac{\cD(k)}2\log{N}\nonumber\\
	&=&\sum\nolimits_{i=1}^k\frac{i+2}2\log{N}+\sum\nolimits_{i=k+1}^d\frac{q+2}2\log{N}.\label{eqn:BIC}
\end{IEEEeqnarray}


\subsection{BIC for factor analysis with incomplete data}\label{sec:bicobs}
To choose $k$ for incomplete data, \cite{Song2008-misfa} suggest using the following form of criterion
\begin{equation}
	\cL^*(k,\hat{\btheta}(k))=\cL_o(\bX_{obs}|\hat{\btheta}(k))-\cP_{bic}(\hat{\btheta}(k)).\label{eqn:incomp.bic}
\end{equation}
Comparing \refe{eqn:cri} with \refe{eqn:incomp.bic}, it can be seen that the complete data log likelihood $\cL$ is now replaced by the observed one $\cL_o$. They have found that the BIC in \refe{eqn:incomp.bic} has good performance. Note that BIC uses the `complete' data sample size $N$ since $\cC(N)=\log{N}$, which means that the penalty terms of BIC are the same no matter whether in incomplete or complete data case. However, for incomplete data, the actual amounts of observed information is only $N_i$'s, as detailed in \refs{sec:intro}. Thus BIC implausibly ignores the amounts of missing information inherent in incomplete data $(N-N_i)'s$. 




Like in complete data case, BIC is usually implemented through a two-stage procedure. Detailedly, given a range of values of $k$ from $k_{min}$ to $k_{max}$, which is assumed to include the optimal one, the two-stage procedure first obtain the ML estimate $\hat{\btheta}(k)$ for each model $k$ and then choose the value
\begin{equation*}
	\hat{k}=\mathop{\hbox{arg\,max}}\limits_{k}\{\cL_o^*(q,\hat{\btheta}(k))\}.\label{eqn:cri1}
\end{equation*}
Given $k$, the ML estimate $\hat{\btheta}(k)$ can be obtained by EM-type algorithms. We give two EM algorithms in \refs{sec:ifa.mle}. 
%

\subsection{Maximum likelihood estimation for factor analysis with incomplete data}\label{sec:ifa.mle}
In this subsection, we develop iterative algorithms to find the ML estimate of parameter $\btheta$ for FA with incomplete data. Let $\bx^o$, $\bmu^o$ and $\bSig^{oo}$ be the observed parts of $\bx$, $\bmu$ and $\bSig$, respectively. By the well known normal result, we have $\bx^o\sim\cN(\bmu^o,\bSig^{oo})$. Then the observed log likelihood of $\btheta=(\bmu, \bA, \bPsi)$ for incomplete data $\bX_{obs}=\{\bx_n^o\}_{n=1}^N$, is
\begin{equation}\label{eqn:mislike}
	\cL_o(\btheta\,|\,\bX_{obs})=\sum\nolimits^N_{n=1}\log{p(\bx^o_n|\btheta)}=-\frac12\sum\nolimits^N_{n=1}\big\{d_n\log{(2\pi)}+\log{|\bSig_n^{oo}|}+(\bx_n^o-\bmu_n^o)'{\bSig_n^{oo}}^{-1}(\bx_n^o-\bmu_n^o)\big\}.
\end{equation}
We should use two EM-type algorithms to maximize $\cL_o$ in \refe{eqn:mislike} because of their simplicity and stability. In \refs{sec:ifa.ecme}, we give an ECME algorithm \citep{chuanhai_ecme}, which is used in our experiments due to its faster computation. In \refs{sec:ifa.ecm}, we give an ECM algorithm \citep{meng-ecm}, which is useful for the development of our proposed criterion.
\subsubsection{The ECME algorithm}\label{sec:ifa.ecme}
In this algorithm, the missing values of $\bx$ are treated as the missing data, namely $\bX_{mis}=\{\bx_n^m\}_{n=1}^N$. The complete data log likelihood of $\btheta$ for complete data $\bX=(\bX_{obs}, \bX_{mis})$ is
\begin{equation}\label{eqn:comlike1} \cL_1(\btheta\,|\bX)=\sum\nolimits^N_{n=1}\log{p(\bx_n|\btheta)}=-\frac12\sum\nolimits^N_{n=1}\left\{d\log{(2\pi)}+\log{|\bSig|}+({\bx_n}-\bmu)'\bSig^{-1}({\bx_n}-\bmu)\right\},
\end{equation}


Let $\bPsi=\mathrm{diag}(\psi_1,\psi_2,
\dots,\psi_d),$ $
\bPsi_i\triangleq\mathrm{diag}(\tpsi_1,
\dots,\tpsi_{i-1},
\psi_i,\psi_{i+1},\dots,\psi_d)$. Given an initial
$\bPsi$, the ECME algorithm that maximizes $\cL_o$ in \refe{eqn:mislike} consists of an E-step and three CM-steps.
\begin{itemize}
	\item \textbf{CML-Step 1:} Given $(\bA,\bPsi)$, maximizing $\cL_o$ in \refe{eqn:mislike} w.r.t. $\bmu$ yields $\tbmu$.
	\item \textbf{E-step:}  Given $\bX_{obs}$ and $(\tbmu, \bA, \bPsi$), compute the expected $\cL_1$ to obtain $Q$ function.
	\item \textbf{CMQ-Step 2:} Given $(\tbmu, \bPsi$), maximizing $Q_1$ w.r.t. $\bA$ yields $\tbA$.
	\item \textbf{CMQ-Step 3:} Given $(\tbmu, \tbA, \bPsi_i$), maximizing $Q_1$ w.r.t. $\psi_i$ yields $\tpsi_i$, sequentially for $i=1,2,\dots,d$.
\end{itemize}
The updating formula for $(\tbmu, \tbA, \tbPsi)$ has been presented in \cite{zhao2014-fa-auto}. For completeness, we also provide them in \refs{sec:app.ecme}.
\subsubsection{The ECM algorithm}\label{sec:ifa.ecm}
Let the complete data be $(\bX_{obs},\bZ)$, where the missing data $\bZ=\{\bz_n\}_{n=1}^N$. In contrast to the ECME algorithm in \refs{sec:ifa.ecme}, the ECM algorithm does not treat the missing observations $\bX_{mis}$ as part of complete data. Similar treatments are also adopted in developing efficient algorithms for fitting FA and the closely related probabilistic principal component analysis (PPCA) on sparse and high-dimensional data \citep{roberts2014factor,Ilin2010-bpca}. Below we give a sketch under this treatment.

From FA model \refe{eqn:fa}, the complete data log likelihood of $\btheta$ for complete data $(\bX_{obs},\bZ)$ is
\begin{equation*}\label{eqn:comlike2} \cL_2(\btheta|\bX_{obs},\bZ)=\sum\nolimits^N_{n=1}\log{p(\bx^o_n,\bz_n|\btheta)}=-\frac12\sum\nolimits^d_{i=1}\sum\nolimits_{O_i}\big\{\log{(2\pi)}+\log{\psi_i}+\frac{1}{\psi_i}(x_{ni}-\ba_i'\bz_n-\mu_i)^2\big\},
\end{equation*}
Given an initial $\btheta$, the ECM algorithm alternates an E-step and two CM-steps.
\begin{itemize}
	\item\textbf{E-step:}  Given $\bX_{obs}$ and $\btheta=(\bmu, \bA, \bPsi)$, compute the expected $\cL_2$ w.r.t. the posterior distribution $p(\bZ|\bX_{obs})$.
\end{itemize}
\begin{IEEEeqnarray*}{rCl}
	Q_2(\btheta)&=&\bbE[\cL_2(\btheta|\bX_{obs},\bZ)|\bX_{obs}]\\
	&=&-\frac12\sum\nolimits^d_{i=1}\sum\nolimits_{O_i}\big\{\log{\psi_i}+\frac{1}{\psi_i}\bbE\left[(x_{ni}-\ba_i'\bz_n-\mu_i)^2|\bx_n^o\right]\big\}+c,\label{eqn:mis.Q2}
\end{IEEEeqnarray*}
where $c$ is a constant. The covariance matrix of the posterior distribution $p(\bz_n|\bx_n^o)$ is given by
\begin{equation*}
	\bSig_{\bz_n}=\left(\sum \nolimits_{i \in O_n} \frac1{\psi_i}\ba_i\ba_i' + \bI\right)^{-1},
\end{equation*}
and the required conditional expectations in \refe{eqn:mis.Q2} $\bbE[\bz_n|\bx_n^o]$ and $\bbE[\bz_n \bz_n'|\bx_n^o]$ are computed by 
\begin{IEEEeqnarray*}{rCl}
	\bbE[\bz_n|\bx_n^o] &=& \bSig_{\bz_n} \sum \nolimits_{i \in O_n} \frac1{\psi_i}\ba_i(x_{ni}-\mu_i), \\
	\bbE[\bz_n \bz_n'|\bx_n^o] &=& \bbE[\bz_n|\bx_n^o] \bbE[\bz_n|\bx_n^o]' + \bSig_{\bz_n}.  
\end{IEEEeqnarray*}
\begin{itemize}
	\item \textbf{CM-Step 1:} Given $(\bA, \bPsi)$, maximize $Q_2$ w.r.t. $\mu_i$ yields $\tmu_i$, $i=1,2,\dots,d$, as follows.
\end{itemize}
\begin{equation}
	\tmu_i = \frac{1}{N_i} \sum \nolimits_{n \in O_i}(x_{ni} - \ba_i'\bbE[\bz_n |x_n^o]).\label{eqn:ecm.mui}
\end{equation}
\begin{itemize}
	\item \textbf{CM-Step 2:} Given $\tbmu$, maximize $Q_2$ w.r.t. $(\ba_i, \psi_i)$ yielding $(\tba_i, \tpsi_i)$, $i=1,2,\dots,d$, as follows.
\end{itemize}
\begin{IEEEeqnarray}{rCl}
	\tba_i &=& \left(\frac1{N_i}\sum \nolimits_{n \in O_i} \bbE[\bz_n \bz_n'|\bx_n^o]\right)^{-1} \frac1{N_i}\sum \nolimits_{n \in O_i} (x_{ni}-\tmu_i) \bbE[\bz_n|\bx_n^o], \label{eqn:ecm.ai}\\
	\tpsi_i &=& \frac{1}{N_i} \sum \nolimits_{n \in O_i}\{(x_{ni}-\tmu_i)^2- (x_{ni}-\tmu_i)\tba_i'\bbE[\bz_n|\bx_n^o]\}.\label{eqn:ecm.psii}
\end{IEEEeqnarray}
\subsubsection{Actual sample size of $\btheta_i=(\mu_i, \ba_i, \psi_i)$}\label{sec:ifa.actual}
From the the ECM algorithm in \refs{sec:ifa.ecm}, we have the following
\begin{obs}\label{obs:ecm.size}
The actual sample size of $\btheta_i$ is $N_i, i=1,2,\dots,d$, rather than the `complete' data sample size $N$.
\end{obs}
From \refe{eqn:ecm.mui}--\refe{eqn:ecm.psii}, it can be observed that  $\btheta_i=(\mu_i, \ba_i, \psi_i)$, which specifies variable $x_i$, is estimated only based on the actual sample size $N_i$, instead of $N$. In particular, if there exits $N_i=0$, it is not possible to estimate the corresponding $\btheta_i$. \refobs{obs:ecm.size} motivates us to develop a new criterion only using the actual amounts of observed information, which extends Approximation 2 under the simple model in \refs{sec:mot} to general cases. 


\section{Novel hierarchical Bayesian information criterion (HBIC) for factor analysis with incomplete data}\label{sec:bic.obs}
Motivated by \refobs{obs:ecm.size} in \refs{sec:ifa.actual}, we propose in \refs{sec:bico1} a novel criterion called hierarchical BIC (HBIC) for model selection in FA model with incomplete data. The novelty is that it only uses the actual amounts of observed information $N_i$'s in the penalty term, rather than the `complete' sample size $N$ taken in the BIC penalty \refe{eqn:incomp.bic}. In \refs{sec:largeN}, we show that HBIC is a large sample limit of variational Bayesian (VB) lower bound. We discuss its relationship with BIC in \refs {sec:rlt.bic}. 

\subsection{The proposed criterion}\label{sec:bico1}
The proposed criterion HBIC also takes a similar form to \refe{eqn:incomp.bic}, consisting of the observed data log likelihood plus a new penalty term
\begin{equation}
	\cL_2^*(k,\hat{\btheta}(k))=\cL_o(\bX_{obs}|\hat{\btheta}(k))-\cP_{hbic}(\hat{\btheta}(k)).\label{eqn:incomp.hbic}
\end{equation}
Here $k$ is the number of factors, $\hat{\btheta}(k)$ denotes the ML estimate of parameter $\btheta$ in $k$-factor model and the penalty term is given by
\begin{IEEEeqnarray}{rCl}
\cP_{hbic}(\hat{\btheta}(k)) &=& \sum\nolimits_{i=1}^d\cP(\hat{\btheta}_i(k))=\sum\nolimits_{i=1}^d\frac{\cD_i(k)}2\log{N_i}\label{eqn:BIC.obs.p}\\
&=& \sum\nolimits_{i=1}^k\frac{i+2}2\log{N_i}+\sum\nolimits_{i=k+1}^d\frac{k+2}2\log{N_i},\nonumber
\end{IEEEeqnarray}
where $\hat{\btheta}_i(k)$ denotes the ML estimate of parameter $\btheta_i$ of variable $x_i$, $N_i$ is actual observed values of variable $x_i$, in the order that $N_1\leq N_2\leq, \dots,\leq N_d\leq N$, and $\cD_i(k)=i+2,i=1,\dots,k;\cD_i(k)=k+2,i=k+1,\dots,d$ is the number of free parameters in $\btheta_i=(\mu_i, \ba_i, \psi_i)$ with $\ba_i$ being the row $i$ of $\bA$ in the form \refe{eqn:vbfa2.A}. 

To use HBIC \refe{eqn:incomp.hbic}, we need to calculate the penalty $\cP_{hbic}(\hat{\btheta}(k))$ using the ascending-ordered $N_i$'s while compute the first term $\cL_o(\bX_{obs}|\hat{\btheta}(k))$ in the conventional way, e.g. using the ECME algorithm in \refs{sec:ifa.ecme}, as the quantity $\cL_o(\bX_{obs}|\hat{\btheta}(k))$ does not depend on the variable order of $\bx$.

From \refe{eqn:BIC.obs.p}, the HBIC penalty at the higher level of model parameter $\btheta$ comprises $d$ BIC penalties at the lower level of parameters $\btheta_i$ for variable $x_i$, using the actual observed sample sizes $N_i$ only. Although this criterion, to our knowledge, is new, the idea to penalize each model parameter only using its relevant sample size, is not completely new. For example, for model selection in mixture models, \cite{Gollini2014-mlta,zhao2014-hbic-mppca} uses a criterion that penalizes the parameter of each component only using its local effective sample size. Similar criteria in the context of hierarchical or random effects models have also been suggested in \cite{Pauler1998-hbic,Raftery2007-hbic}.

\subsection{Hierarchical BIC (HBIC): large sample limit of a lower bound on the marginal likelihood}\label{sec:largeN}

\subsubsection{Bayesian approaches to FA and the marginal likelihood}\label{sec:b.alg}
Given incomplete data $\bX_{obs}$, the objective in the Bayesian treatment of FA model is to evaluate the posterior distribution
$$p(k|\bX_{obs})\propto p(\bX_{obs}|k)p(k).$$ When there is no information other than the data $\bX_{obs}$, each model $k$ is generally assumed to be equally likely a priori. In this case, the interesting term is the marginal likelihood or model evidence $p(\bX_{obs}|k)$, which is obtained by integrating over the parameter space of $\btheta$
\begin{equation}
p(\bX_{obs})=\int{p(\bX|\btheta,k)p(\btheta)}d\btheta.\label{eqn:ifa.mgn.l}
\end{equation}
Here, $p(\btheta)$ is a prior distribution over parameter $\btheta=(\bmu, \bA,
\bvarphi)$, where $\bvarphi=\bPsi^{-1}$. For notation convenience, we omit the dependence on the $k$-factor model.

For FA model, it is computationally and analytically intractable to perform the integral \refe{eqn:ifa.mgn.l} exactly. In the case of complete data $\bX$, \cite{Lopes2004-bfa} propose a fully Bayesian learning algorithm by means of the computationally intensive sampling-based Markov Chain Monte Carlo (MCMC) method, which results in a stochastic approximation solution to the marginal likelihood $p(\bX)$. 

Unlike the ML method in \refs{sec:ifa.mle}, where the objective function $\cL_o$ in \refe{eqn:mislike} is the observed data log likelihood given parameter $\btheta$, the objective of Bayesian methods is the marginal likelihood, which integrates out parameters and can automatically penalize the model with more degrees of freedom \citep{Beal2003-phd}.

\subsubsection{Variational Bayesian learning algorithm and its lower bound on the marginal likelihood}\label{sec:vb.alg}
Variational Bayesian (VB) methods originate from machine learning community \citep{bishop-prml,wainwright2008graphical}, but have also become increasing popular in statistics community \citep{blei2017variational}. Compared with the sampling-based Bayesian methods in \refs{sec:b.alg}, VB methods trun the problem into an optimization problem, which yields a deterministic approximation solution to the marginal likelihood $p(\bX)$ and hence are computationally more efficient \citep{bishop-prml}. Some recent works include VB inference for variable selection in logistic regression models \citep{zhang2019novel}, VB inference for network autoregression models \citep{lai2022variational} and etc. For FA model with complete data $\bX$, \cite{zhao2009-vbfa} propose a VB learning algorithm. In the case of incomplete data $\bX_{obs}$, \cite{Ilin2010-bpca} propose several VB approximation solutions to the marginal likelihood $p(\bX_{obs})$ for the closely related PPCA model. 

Below we give a simple derivation of VB for FA with incomplete data, which is generally similar to that for PPCA model in \cite{Ilin2010-bpca}. By Jensen's inequality, the log marginal likelihood can be bounded by
\begin{IEEEeqnarray}{rCl}
	\log{p(\bX_{obs})}&=&\log{\int{p(\bX_{obs}, \bZ, \btheta)}}d\bZ d\btheta\nonumber\\
	& &\geq\int{q(\bZ, \btheta)\log{\frac{p(\bX_{obs}, \bZ, \btheta)}{q(\bZ, \btheta)}}}d\bZ
	d\btheta=\cF(q),\label{eqn:VB.bound}
\end{IEEEeqnarray}
where $q(\bZ, \btheta)$ is a free distribution of latent factors $\bZ$ and parameter $\btheta$. The difference between $\log{p(\bX_{obs})}$ and $\cF(q)$ in \refe{eqn:VB.bound} can also be expressed in terms of Kullback-Leibler (KL) divergence.
\begin{IEEEeqnarray}{rCl}
	\hbox{KL}(q||p)&=&\log{p(\bX_{obs})}-\cF(q)\nonumber\\
	&=&-\int{q(\bZ, \btheta)\log{\frac{p(\bZ, \btheta|\bX_{obs})}{q(\bZ,
				\btheta)}}}d\bZ  d\btheta\label{eqn:KL}.
\end{IEEEeqnarray}
It can be seen that maximizing $\cF$ in \refe{eqn:VB.bound} is
equivalent to minimizing the KL divergence \refe{eqn:KL} between $q(\bZ,
\btheta)$ and the true posterior $p(\bZ, \btheta|\bX_{obs})$. From \refe{eqn:KL}, the KL divergence is minimized when $q(\bZ, \btheta)=p(\bZ, \btheta|\bX_{obs})$, which we then substitute into \refe{eqn:VB.bound}. This leads to the equality $\log{p(\bX_{obs})}=\cF(q)$. However, this fails to simplify the problem as the true posterior $p(\bZ, \btheta|\bX_{obs})$ requires knowing the normalizing constant, namely the analytically intractable marginal likelihood $p(\bX_{obs})$. Instead, VB approaches this problem by utilizing a simpler factorized distribution $q(\bZ, \btheta)\approx q(\bZ)q(\btheta)$ to approximate $p(\bZ, \btheta|\bX_{obs})$ and aims to optimize a lower bound $\cF$ of the log marginal likelihood $\log{p(\bX_{obs})}$. The bound $\cF(q(\bZ),q(\btheta))$ is a functional of the free distributions $q(\bZ)$ and $q(\btheta)$.



Since $\bA$ in FA model can only be determined up to a rotation, we follow \cite{zhao2009-vbfa} to use
the following lower triangular matrix for $\bA$:
\begin{equation}\label{eqn:vbfa2.A}
\bA=\left(\begin{matrix}
a_{11}&0&0&\cdots&0&0\\
a_{21}&a_{22}&\color{red}{0}&\dots&0&0\\
\vdots&\vdots&\vdots&\ddots&\vdots&\vdots\\
a_{q-1,1}&a_{q-1,2}&a_{q-1,3}&\cdots&a_{q-1, q-1}&\color{red}{0}\\
a_{q,1}&a_{q,2}&a_{q,3}&\cdots&a_{q, q-1}&a_{q, q}\\
\vdots&\vdots&\vdots&\ddots&\vdots&\vdots\\
a_{d,1}&a_{d,2}&a_{d,3}&\cdots&a_{d, q-1}&a_{d, q}
\end{matrix}\right).
\end{equation}
Clearly, \refe{eqn:vbfa2.A} reduces the of free parameters of $\bA$ by $q(q-1)/2$ directly. However, different forms of $\bA$ that eliminates the rotation could yield different VB lower bounds \citep{zhao2009-vbfa}. Intuitionally, a row of $\bA$ containing more number of free parameters should be estimated by more number of observed values $N_i$ and thus \refe{eqn:vbfa2.A} is more suitable for the case $N_1\leq N_2\leq, \dots,\leq N_d\leq N$. If $N_i$'s are not in this ascending order, then $(x_1,\dots,x_d)$ can be rearranged so that $N_i$'s are in the ascending order. Therefore, in what follows we assume that $N_i$'s are in the ascending order. 

Let $\btheta_i=\{\mu_i,\ba_i,\varphi_i\}_{i=1}^d$ and $\ba_i$ be a column vector corresponding to the row $i$ of
$\bA$ (given by \refe{eqn:vbfa2.A} if $N_1\leq N_2\leq, \dots,\leq N_d\leq N$). As shown in \cite{zhao2009-vbfa}, if we use the prior $p(\btheta)=\prod_{i=1}^dp(\btheta_i)$, where the priors of $\btheta_i$ are independent of each other, we can obtain the factorization 
\begin{equation}
  q(\bZ, \btheta)=q(\bZ)q(\btheta)=q(\bZ)\prod\nolimits_{i=1}^dq(\btheta_i),\label{eqn:vb.factor}
\end{equation}
where the additional factorization $q(\btheta)=\prod_{i=1}^dq(\btheta_i)$ is called \emph{induced} factorization in \cite{bishop-prml}. Following \cite{Ilin2010-bpca,zhao2009-vbfa}, substitute \refe{eqn:vb.factor} into
\refe{eqn:VB.bound} and maximize $\cF$ over the distributions $q(\bZ)$ and $q(\btheta_i)$, $i=1,\cdots,d$, leading to the following VBEM updating steps:
\begin{itemize}
  \item VBE-step:
  \begin{equation}\label{eqn:vb.q.z}
    q(\bZ)=\sum\nolimits^N_{n=1}\log{q(\bz_n)}=\sum\nolimits^N_{n=1}\bbE[\log{p(\bx^o_n, \bz_n,
    \btheta)}]_{q(\btheta)}+c.
  \end{equation}
\end{itemize}
\begin{itemize}
  \item VBM-steps:
\begin{equation}
\log{q(\btheta_i)}=\sum\nolimits_{O_i}\bbE[\log{p(\bx^o_{ni}, \bz_n,\btheta_i)}]_{q(\bz_n)}+c, \label{eqn:q.thetai}
\end{equation}
\end{itemize}
where $i=1,\cdots,d$ and $c$ is a constant. Since \refe{eqn:vb.q.z} and
\refe{eqn:q.thetai} are coupled, the VBEM algorithm alternatively
iterates \refe{eqn:vb.q.z} and \refe{eqn:q.thetai} until convergence. 
From the VBEM algorithm in \refs{sec:vb.alg}, we have the following
\begin{obs}\label{obs:vbem.size}
	The actual sample size for the distribution $q(\btheta_i)$ is $N_i, i=1,2,\dots,d$, instead of the `complete' data sample size $N$.
\end{obs}
\refobs{obs:vbem.size} can be seen from \refe{eqn:q.thetai}, which is consistent with that in the ECM algorithm for ML estimate \refs{sec:ifa.actual}. In fact, the VBEM algorithm can reduce to the ECM algorithm in \refs{sec:ifa.ecm} if the parameter density $q(\btheta)$ is restricted to be the Dirac delta function as $q(\btheta)=\delta(\btheta-\btheta^*)$ \citep{bernardo2003variational}. In \refs{sec:F.largeN}, \refobs{obs:vbem.size} will be utilized to find the large sample limit of the bound $\cF$.

\subsubsection{Large sample limit of the lower bound on the marginal likelihood}\label{sec:F.largeN}
Substituting \refe{eqn:vb.factor} into
\refe{eqn:VB.bound}, we can write $\cF$ as
\begin{equation}\label{eqn:VB.bound.1}
  \cF=\underbrace{\bbE\left[\log{\frac{p(\bX_{obs},
  \bZ|\btheta)}{q(\bZ)}}\right]_{q(\bZ)q(\btheta)}}_{\cF_{\cD}}-\underbrace{\sum\nolimits_{i=1}^d\hbox{KL}[q(\btheta_i)||p(\btheta_i)]}_{\cF_{P}}.
\end{equation}

Let us consider the large sample limits of the two terms $\cF_{\cD}$ and $\cF_{P}$. We begin with $\cF_{P}$. It is proved in \cite{Beal2003-phd}
that under mild conditions, the variational posterior
distribution $q(\btheta)$ for exponential family models is approximately normal. Since FA model is a member in exponential family models and $q(\btheta)=\prod_{i=1}^dq(\btheta_i)$, it follows that $q(\btheta_i)$ is also approximately normal. This means that the KL divergence $\hbox{KL}[q(\btheta_i)||p(\btheta_i)]$ in $\cF_{P}$ can be calculated using a Guassian approximation. Thus we have that as
$N\rightarrow\infty$
\begin{equation*}
  \bbE[\log{q(\btheta_i)}]\approx\frac12\log{|2\pi H_i(\hat{\btheta}_i)|},\label{eqn:q.theta.lmt}
\end{equation*}
where $H_i(\hat{\btheta}_i)$ is the negative Hessian matrix of
$q(\btheta_i)$ evaluated at ML estimate $\hat{\btheta}_i$. Using
the fact that $H_i$ scales linearly with $N_i$, the actual sample size of $\btheta_i$ as shown in \refobs{obs:vbem.size}, we have that as
$N\rightarrow\infty$, $H_i/N_i$ converges to a constant
matrix, denoted by $H_{i0}$, and
\begin{eqnarray}
  \bbE[\log{q(\btheta_i)}]&=&\frac{\cD_i(k)}2\log{N_i}+\frac12\log{|2\pi
  H_{i0}|}+\cO(1)\nonumber\\
  &=&\frac{\cD_i(k)}2\log{N_i}+\cO(1), \label{eqn:q.ta.lmt.2}
\end{eqnarray}
where $\cD_i(k)$ is the number of free parameters of $\btheta_i$. For example, for $\bA$ in \refe{eqn:vbfa2.A}, $\cD_i(k)=i,i=1,\dots,k$, $\cD_i(k)=k,i=k+1,\dots,d$, and $\sum_{i=1}^d\cD_i(k)=d(k+2)-k(k-1)/2=\cD(k)$. From \refe{eqn:q.ta.lmt.2}, we obtain that, as $N\rightarrow\infty$,
\begin{eqnarray}
  \cF_{P}&=&\sum\nolimits_{i=1}^d\left(\bbE[\log{q(\btheta_i)}]-\bbE[\log{p(\btheta_i)}]\right)\nonumber\\
  &=&\sum\nolimits_{i=1}^d\left(\frac{\cD_i(k)}2\log{N_i}-\log{p(\hat{\btheta}_i)}\right)+\cO(1), \label{eqn:FP.lmt}
\end{eqnarray}


Next we analyze $\cF_{\cD}$. Since $q(\btheta_i)$ is approximately normal, it follows that, as $N\rightarrow\infty$, $q(\btheta_i)$ will be strongly peaked at $\hat{\btheta}_i$ \citep{Attias1999,Beal2003-phd} and we have $$q(\btheta_i)=\delta(\btheta_i-\hat{\btheta}_i)+\cO(1),$$ where $\delta(\cdot)$ is the Dirac delta function. Combining this result with \refe{eqn:vb.q.z}, we have $$q(\bz_n)=p(\bz_n|\bx^o_n,\hat{\btheta})+\cO(1), \hbox{as}\,\, N\rightarrow\infty.$$ Substituting this result into $\cF_D$, we obtain that
\begin{equation}\label{eqn:FD.lmt}
\cF_{\cD}=p(\bX_{obs}|\hat{\btheta})+\cO(1).
\end{equation}
Substituting \refe{eqn:FD.lmt} and \refe{eqn:FP.lmt} into \refe{eqn:VB.bound.1}, and further dropping the order-1 term $\sum\nolimits_{i=1}^d\log{p(\hat{\btheta}_i)}$, we obtain the following \refthm{thm:F.limit}.
\begin{thm}\label{thm:F.limit}
As the sample size $N\rightarrow\infty$, the VB lower bound in \refe{eqn:VB.bound.1}
  \begin{equation*}
  \cF=\cL(\bX_{obs}| 
\hat{\btheta})-\cP_{hbic}(\hat{\btheta})+O(1)\label{eqn:F.limit},
  \end{equation*}
  where $\cP_{hbic}(\hat{\btheta})$ is given by \refe{eqn:BIC.obs.p}.
\end{thm}
With \refthm{thm:F.limit}, we obtain the HBIC criterion as detailed in \refs{sec:bico1}.

\subsection{Relationship with BIC and HBIC}\label{sec:rlt.bic}
For incomplete data, it is easy to see from \refe{eqn:BIC} and \refe{eqn:BIC.obs.p} that BIC penalizes the model parameter $\btheta_i$ using the whole sample size $N$ while HBIC using the actual observed sample sizes $N_i$'s only. In fact, since $N_i\leq N$, HBIC penalizes model lighter than BIC. Despite the difference, there exists close relationship between HBIC and BIC. For complete data, $N_i=N,i=1,\dots,d$, HBIC degenerates into BIC. Let $\gamma_i$ stands for the missing rate of variable $i$. Under fixed missing rates for incomplete data, we have the following \refthm{thm:rlt}.
%
\begin{thm}\label{thm:rlt}
	Under the assumption that the missing rates $\gamma_i$'s are constants, as the sample size $N\rightarrow\infty$, HBIC is consistent since BIC is a further
	approximation of HBIC.
\end{thm}
\begin{proof}
Since $N_i=N(1-\gamma_i)$, substitute this equality into \refe{eqn:BIC.obs.p} and dropping the order-1 term $\sum\nolimits_{i=1}^ki\log{(1-\gamma_i)}+\sum\nolimits_{i=k+1}^dk\log{(1-\gamma_i)}/2$ that does not increase with $N$, we obtain that BIC emerges as a further
	approximation of HBIC as $N\rightarrow\infty$. This complete the proof.
\end{proof}
\refthm{thm:rlt} shows that HBIC shares the theoretical consistency of BIC. In addition, it should be noted that, under the simple model in \refs{sec:mot}, HBIC \refe{eqn:BIC.obs.p} degenerates to Approximation 2 of BIC. 

\section{Experiments}\label{sec:expr}
The theoretical analysis in \refs{sec:bic.obs} shows that BIC is a further
approximation of HBIC in large sample limit and HBIC penalizes model lighter than BIC. However, this analysis fails to tell us whether HBIC could be more advantageous than BIC. We hence empirically compare HBIC and BIC with synthetic and real-world data sets. For a fair comparison, we use the two-stage procedure described in \refs{sec:bicobs}. The ECME algorithm described in \refs{sec:ifa.ecm} is used to obtain the ML estimate. After using mean imputation to fill the missing values \citep{little_samd}, we initialize the parameters with the corresponding PCA starting value from the sample covariance matrix. In addition, we set $\eta=0.005$ and stop the algorithm when the relative change in the actual data log likelihood is smaller than a threshold $tol=10^{-8}$ or the number of iterations exceeds a maximum $500$.


\subsection{Synthetic data}\label{sec:simu}
In this section, we use synthetic data to compare the performance of BIC, HBIC, AIC and CAIC. Let
\begin{equation}\label{eqn:simu.para}
	\begin{split}
\bmu&=\bo_d,\,\bPsi=\diag\{\mbox{linspace}(0.9,1,d)\},\\
\bA'&=\begin{pmatrix}
	0.8&0.6&0&0&0&0&0&0&0&0\\
	0&0.1&0.7&-0.7&0.8&0&0.2&-0.1&0.1&-0.1\\
	0&0&-0.1&0.1&0.1&0.9&0.95&-0.95&-0.8&-0.95
\end{pmatrix}.
\end{split}
\end{equation}
To reduce variability, we generate 100 training data sets and report the successful, underestimating, and overestimating rates in choosing the number of factors.

\subsubsection{Low-dimensional data}\label{sec:expr.low}
The objective of this experiment is to compare the finite sample performance of BIC and HBIC on a low-dimensional data with various missing rates. To this end, we generate a $d=10$-dimensional dataset with sample size $N=250$ from a $k=3$-factor model \refe{eqn:fa} parameterized by $(\bmu_1,\bA_1,\bPsi_1)$ as
\begin{equation*}
	\bmu_1=\bmu,\,\bPsi_1=0.1\cdot\bPsi, \bA_1=\bA,
\end{equation*}
where $\bmu,\bPsi,\bA$ are given by \refe{eqn:simu.para}. \reff{fig:low}~(a) shows a typical scree plot obtained from one complete dataset. It can be seen that both the Kaiser's rule and scree plot suggest choosing three factors, which is the true number. 

To investigate the performance in incomplete data cases with various missing rates. We create the incomplete data set by deleting randomly the observed values of each variable $x_i$ according to the missing rate vector $\bgamma_1$. Specifically, we set  $\bgamma_1=(0.6m\cdot \bl_2;\,\, 0.7m\cdot\bl_3;\,\,0.1\cdot\bl_5)$, where $m=0,0.95,1,1.05$. The value of $m$ affects the missing rate. The higher the value of $m$, the higher the missing rate. The case of $m=0$ is close the complete data case. \reff{fig:low}~(b)--(e) show typical evolvements of criterion values versus number of factors by the four criteria with various values of missing rates. \reft{tab:low} summarizes the results of underestimating ($U$), successful ($S$), and overestimating ($O$) rates over 100 replications. The main observations includes
\begin{enumerate}[(i)]
	\item When missing rate is small, e.g., $m=0$, all of the BIC, CAIC and HBIC pick the correct number of factors while AIC suffers from overestimation. 
	\item When missing rate gets larger, i.e., $m$ from 0.95 to 1.05, HBIC performs more accurately than BIC, while CAIC suffers from the most serious underestimation.
\end{enumerate}
\begin{figure}[tbh]
	\centering 
	\subfigure{
		\label{fig:low-1}\scalebox{0.5}[0.5]{\includegraphics*{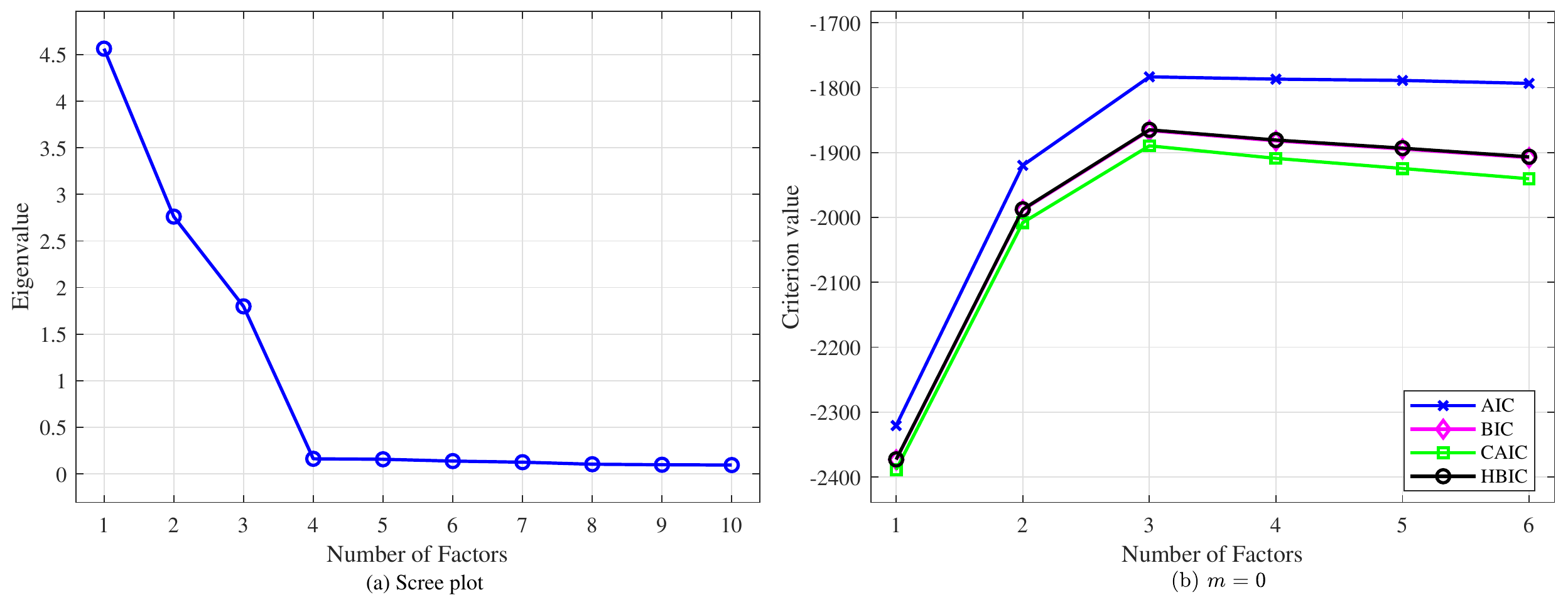}}}
	\subfigure{
		\label{fig:low-2}\scalebox{0.5}[0.5]{\includegraphics*{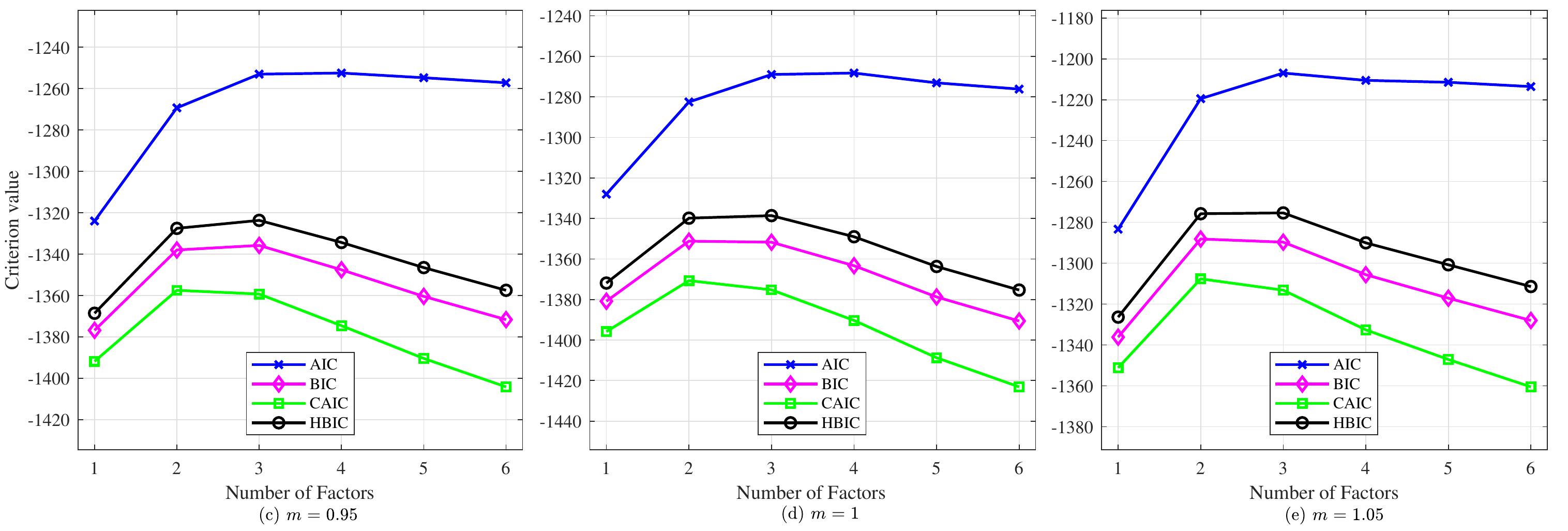}}}
	\caption{Typical evolvements of criterion values versus number of factors by different criteria with various missing rates on a low-dimensional dataset.} 
	\label{fig:low}
\end{figure}
\begin{table}[tbh]
	\centering
	\caption{Results of AIC, BIC, CAIC and HBIC on low-dimensional data sets with various missing rates over 100 replications: Rates of underestimation ($U$), success ($S$), and overestimation ($O$).}
	\label{tab:low}
	\begin{tabular}{ccccccccccccc}
		\toprule 
		\multirow{2}{*}{Criterion} & \multicolumn{3}{c}{$m=$ 0} & \multicolumn{3}{c}{$m=$ 0.95} & \multicolumn{3}{c}{$m=$ 1} & \multicolumn{3}{c}{$m=$ 1.05} \\
		\cmidrule(l){2 - 4}  \cmidrule(l){ 5 - 7 }  \cmidrule(l){ 8 - 10 }  \cmidrule(l){ 11 - 13 } 
		~ &{$U$} & \multicolumn{1}{c}{$S$} & \multicolumn{1}{c}{$O$} & \multicolumn{1}{c}{$U$} & \multicolumn{1}{c}{$S$} & \multicolumn{1}{c}{$O$} & \multicolumn{1}{c}{$U$} & \multicolumn{1}{c}{$S$} & \multicolumn{1}{c}{$O$} & \multicolumn{1}{c}{$U$} & \multicolumn{1}{c}{$S$} & \multicolumn{1}{c}{$O$} \\
		\midrule 
		AIC & 0 & 90 & 10 & 0 & 82 & 18 & 0 & 79 & 21 & 0 & 82 & 18 \\
		BIC & 0 & 100 & 0 & 4 & 96 & 0 & 14 & 86 & 0 & 29 & 71 & 0 \\
		CAIC & 0 & 100 & 0 & 19 & 81 & 0 & 29 & 71 & 0 & 53 & 47 & 0 \\
		HBIC & 0 & 100 & 0 & 2 & 98 & 0 & 7 & 93 & 0 & 19 & 81 & 0 \\
		\bottomrule
	\end{tabular}
\end{table}

\subsubsection{High-dimensional data}\label{sec:expr.high}
In this experiment, we further compare the four criteria using a dataset with higher dimensionality. We generate a $d=40$-dimensional dataset with sample size $N=400$ from a $k=6$-factor model \refe{eqn:fa} parameterized by $(\bmu_2,\bA_2,\bPsi_2)$ as
\begin{equation*}
	\bmu_2=\bmu,\,\bPsi_2=0.2\cdot\bPsi,\, \bA_1=(\bA;\bA);\, \bA_2=\mbox{blkdiag}(\bA_1,\bA_1);
\end{equation*}
where $\bmu,\bPsi,\bA$ are given by \refe{eqn:simu.para}. \reff{fig:high}~(a) shows a typical scree plot obtained from one complete dataset. It can be seen that both the Kaiser's rule and scree plot suggest choosing six factors, which is the true number. 

To investigate the performance in incomplete data cases with various missing rates. 
We set $\bgamma_2=(\bgamma_1;\bgamma_1;\bgamma_1;\bgamma_1)$, where $m=0,0.95,1,1.05$. \reff{fig:high}~(b)--(e) show typical evolvements of criterion values versus number of factors by the four criteria with various values of missing rates. \reft{tab:high} summarizes the results of underestimating ($U$), successful ($S$), and overestimating ($O$) rates over 100 replications. It can be seen that the observations are generally consistent with those in \refs{sec:expr.low}.
\begin{figure}[tbh]
	\centering 
	\subfigure{
		\label{fig:high-1}\scalebox{0.5}[0.5]{\includegraphics*{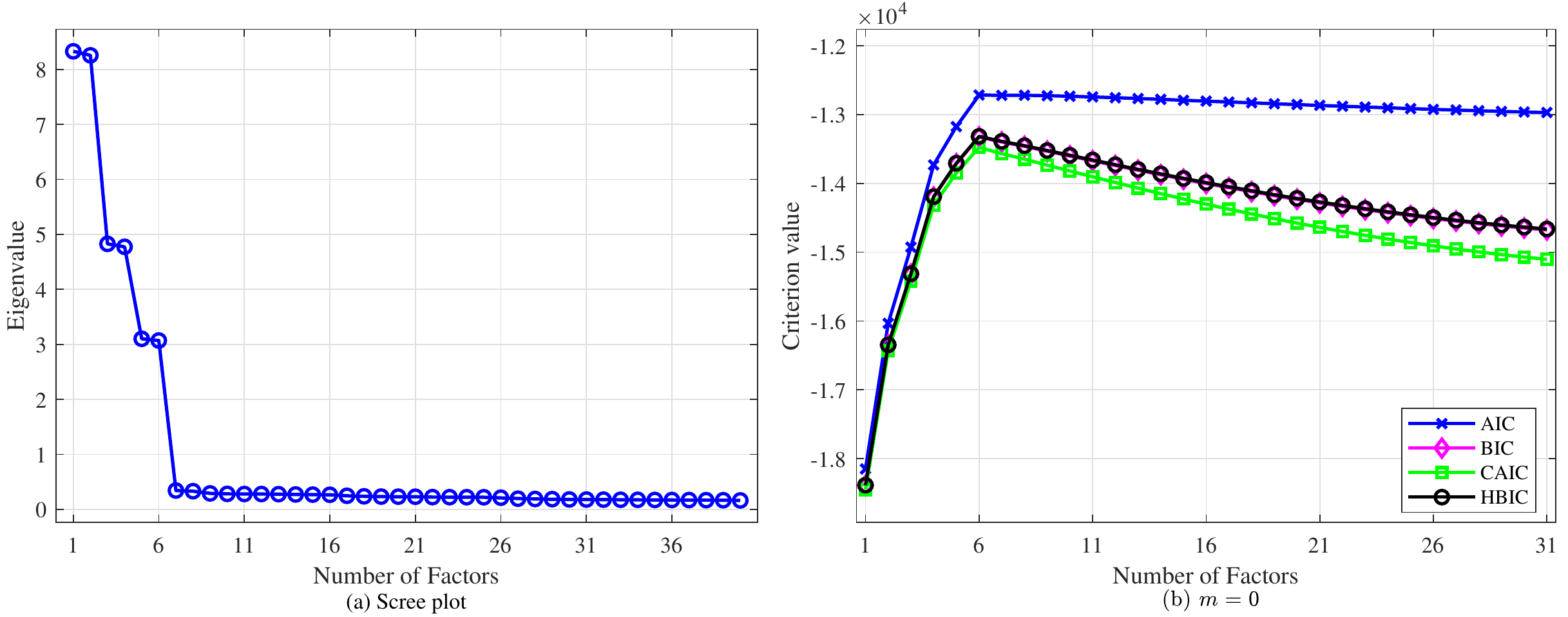}}}
	\subfigure{
		\label{fig:high-2}\scalebox{0.5}[0.5]{\includegraphics*{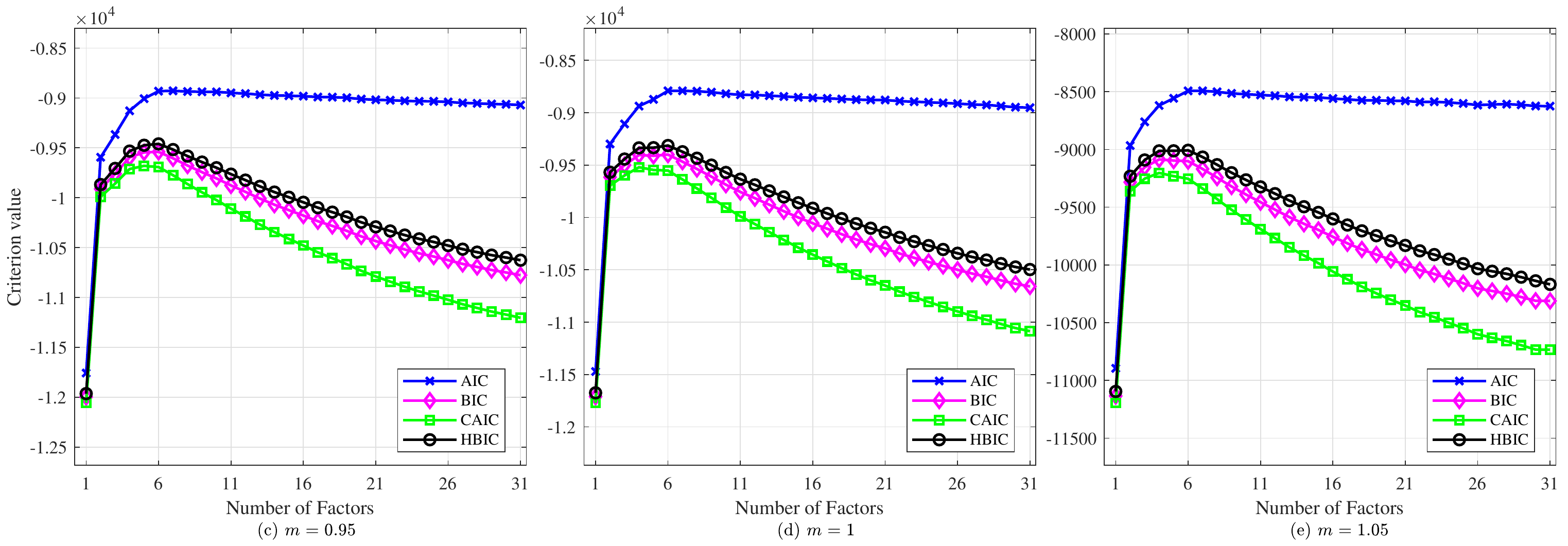}}}
	\caption{Typical evolvements of criterion values versus number of factors by different criteria with various missing rates on a high-dimensional dataset.} 
	\label{fig:high}
\end{figure}
\begin{table}[tbh]
	\centering
	\caption{Results of the four criteria on high-dimensional datasets with various missing rates over 100 replications: Rates of underestimation ($U$), success ($S$), and overestimation ($O$).}
	\label{tab:high}
	\begin{tabular}{ccccccccccccc}
		\toprule 
		\multirow{2}{*}{Criterion} & \multicolumn{3}{c}{$m=$ 0} & \multicolumn{3}{c}{$m=$ 0.95} & \multicolumn{3}{c}{$m=$ 1} & \multicolumn{3}{c}{$m=$ 1.05} \\
		\cmidrule(l){2 - 4} \cmidrule(l){5 - 7} \cmidrule(l){8 - 10}  \cmidrule(l){11 - 13} 
		~ &\multicolumn{1}{c}{$U$} & \multicolumn{1}{c}{$S$} & \multicolumn{1}{c}{$O$} & \multicolumn{1}{c}{$U$} & \multicolumn{1}{c}{$S$} & \multicolumn{1}{c}{$O$} & \multicolumn{1}{c}{$U$} & \multicolumn{1}{c}{$S$} & \multicolumn{1}{c}{$O$} & \multicolumn{1}{c}{$U$} & \multicolumn{1}{c}{$S$} & \multicolumn{1}{c}{$O$} \\
		\midrule 
		AIC & 0 & 81 & 19 & 0 & 70 & 30 & 0 & 62 & 38 & 0 & 56 & 44 \\
		BIC & 0 & 100 & 0 & 0 & 100 & 0 & 5 & 95 & 0 & 33 & 67 & 0 \\
		CAIC & 0 & 100 & 0 & 12 & 88 & 0 & 43 & 57 & 0 & 79 & 21 & 0 \\
		HBIC & 0 & 100 & 0 & 0 & 100 & 0 & 2 & 98 & 0 & 11 & 89 & 0 \\
		\bottomrule
	\end{tabular}
\end{table}

\subsection{Real data}\label{sec:real}

In this experiment, we use the cereal dataset \citep{Lattin2003-mult} to further compare the performance of BIC and HBIC. The dataset is obtained from a survey by 116 cereal consumers on 12 popular cereal brands. It comprises 235 observations, evaluated on 25 variables. \cite{Lattin2003-mult,zhao2014-fa-auto} have analyzed this dataset and found that both BIC and AIC choose a 4-factor model for the complete dataset or incomplete dataset with 5\% missing. 

In this experiment, we are interested in comparing BIC and HBIC when the data suffers from a higher missing rate. To this end, we use a subset consisting of ten variables: Filling, Natural, Fibre, Health, Nutritious, Sweet, Salt, Sugar, Kids, and Family. The sample size $N$ is still 235. \reff{fig:real}~(a) shows the scree plot obtained from the complete dataset. For this 10-variable subset, it can be seen from \reff{fig:real}~(a) that both the Kaiser's rule and scree plot suggest choosing three factors. 

To investigate the performance in incomplete data cases with various missing rates. We create the incomplete data set by randomly deleting the observed values of each variable $x_i$ according to the missing rate vector $\bgamma_3$. We perform 100 replications, and hence we have 100 incomplete datasets for each missing rate $\bgamma_3$ under consideration. Specifically, we set $\bgamma_3=(0.4\cdot\bl_5;\,\, 0.5m\cdot\bl_5)$, where $m=0,0.9,1,1.1$. \reff{fig:real}~(b)--(e) show typical evolvements of criterion values versus number of factors by the four criteria with various values of missing rates. The detailed results over 100 replications are summarized in \reft{tab:cereal2}, including the rates of underestimating ($U$), successful ($S$), and overestimating ($O$). It can be seen from \reft{tab:cereal2} that  
\begin{enumerate}[(i)]
	\item when $m=0$, AIC suffers from overestimation while all of the BIC, CAIC, and HBIC pick the model with $k=3$, which is as expected since this case is closer the complete data case.
	\item When missing rate gets larger, i.e., $m$ from 0.9 to 1.1, HBIC performs more accurately than BIC, while CAIC suffers from the most serious underestimation. 
\end{enumerate}
These observations are generally consistent with those in \refs{sec:simu}.
\begin{figure}[tbh]
	\centering 
	\subfigure{
		\label{fig:real-1}\scalebox{0.5}[0.5]{\includegraphics*{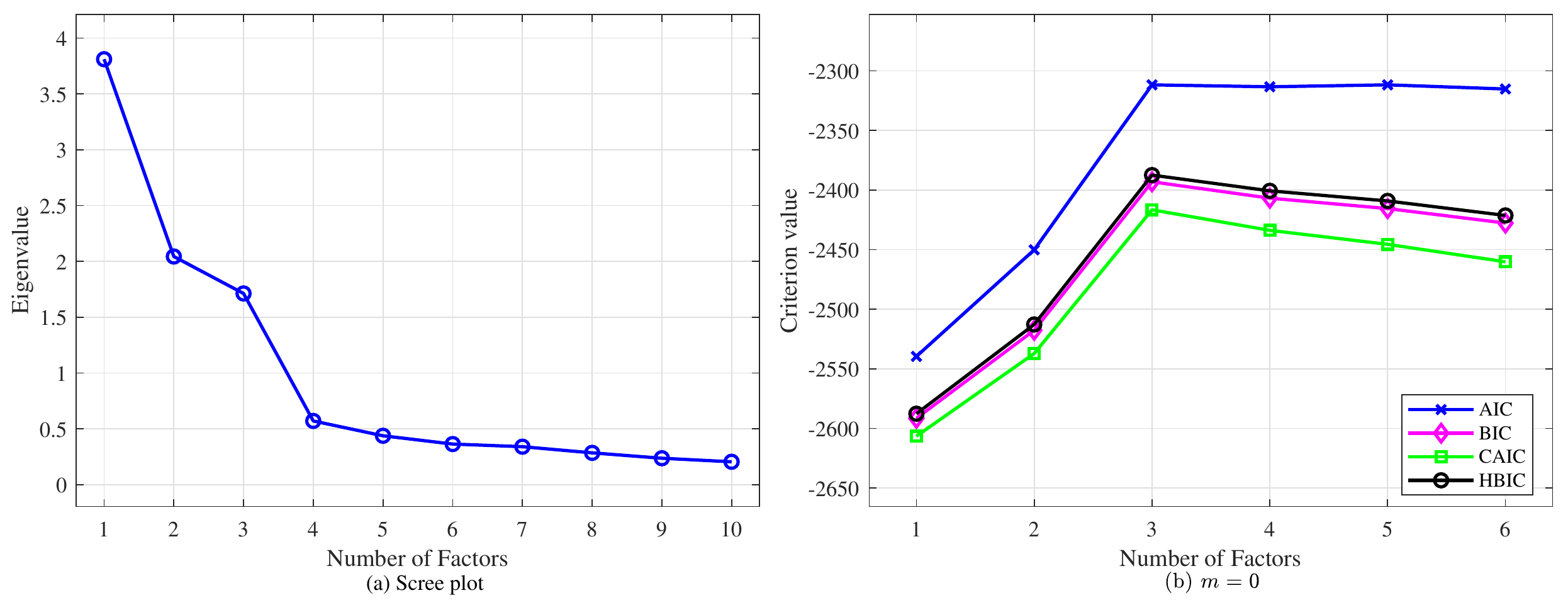}}}
	\subfigure{
		\label{fig:real-2}\scalebox{0.5}[0.5]{\includegraphics*{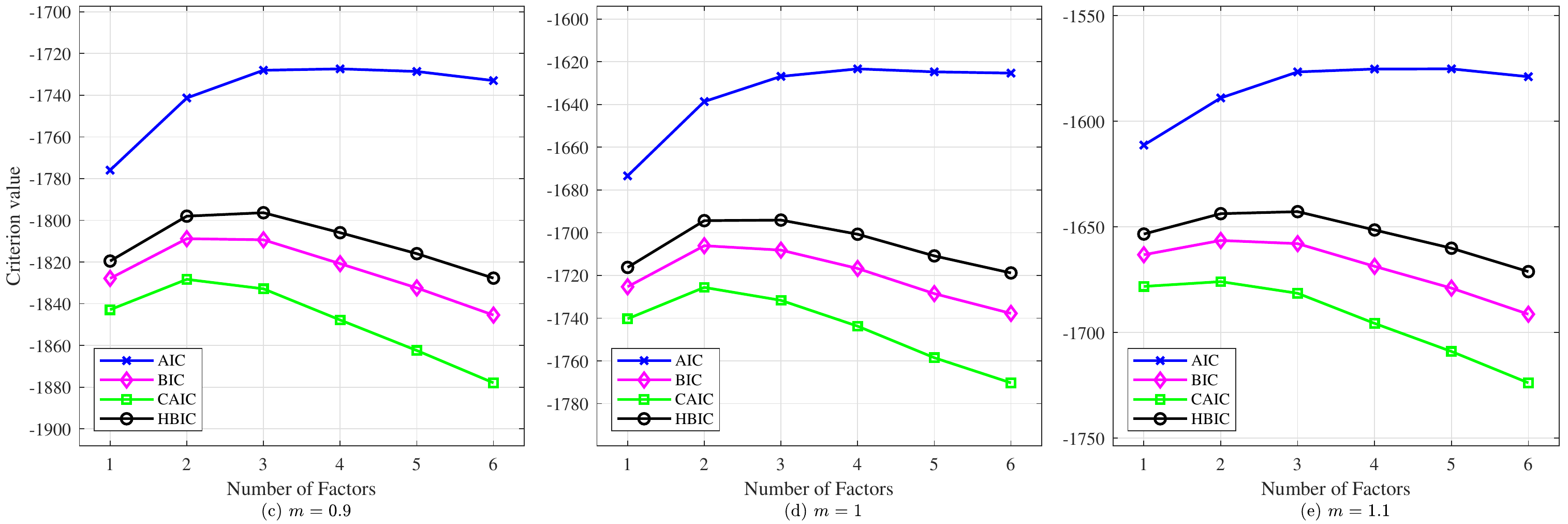}}}
	\caption{Typical evolvements of criterion values versus number of factors by different criteria with various missing rates on cereal data.}		 
	\label{fig:real}
\end{figure}
\begin{table}[tbh]
	\centering
	\caption{Results of AIC, BIC, CAIC and HBIC on cereals data with various missing rates over 100 replications: Rates of underestimation ($U$), success ($S$), and overestimation ($O$).}
	\label{tab:cereal2}
	\begin{tabular}{ccccccccccccc}
		\hline \multirow{2}{*}{ Criterion } & \multicolumn{3}{c}{$m=0$} & \multicolumn{3}{c}{$m=0.9$} & \multicolumn{3}{c}{$m=1$} & \multicolumn{3}{c}{$m=1.1$} \\
		\cline { 2 - 4 }  \cline { 5 - 7 }  \cline { 8 - 10 }  \cline { 11 - 13 } 
		 ~ &\multicolumn{1}{c}{$U$} & \multicolumn{1}{c}{$S$} & \multicolumn{1}{c}{$O$} & \multicolumn{1}{c}{$U$} & \multicolumn{1}{c}{$S$} & \multicolumn{1}{c}{$O$} & \multicolumn{1}{c}{$U$} & \multicolumn{1}{c}{$S$} & \multicolumn{1}{c}{$O$} & \multicolumn{1}{c}{$U$} & \multicolumn{1}{c}{$S$} & \multicolumn{1}{c}{$O$} \\
		\hline 
		AIC & 0 & 69 & 31 & 0 & 72 & 28 & 0 & 73 & 27 & 0 & 65 & 35 \\
		BIC & 0 & 100 & 0 & 3 & 97 & 0 & 14 & 86 & 0 & 31 & 69 & 0 \\
		CAIC & 0 & 100 & 0 & 7 & 93 & 0 & 29 & 71 & 0 & 52 & 48 & 0 \\
		HBIC & 0 & 100 & 0 & 0 & 100 & 0 & 7 & 93 & 0 & 20 & 80 & 0 \\
		\hline
	\end{tabular}
\end{table}

\section{Concluding remarks}\label{sec:cons}
We have developed a new criterion called hierarchical BIC (HBIC) for model selection in factor analysis (FA) model with incomplete data. Unlike BIC, which penalizes the model complexity using `complete' sample size, HBIC does this only using the actual amounts of observed information $N_i$'s. Experiments with incomplete synthetic and real data show that the proposed HBIC can be more accurate than BIC, particularly when the missing rate is not small. Therefore, HBIC is a better criterion than BIC for incomplete data.

For future work, it would be interesting to investigate how to extend the proposed HBIC to the FA-related models in the presence of incomplete data such as mixtures of factor analyzers (MFA) \citep{wang2020automated}, and mixtures of common FA (MCFA) \citep{Baek2010-mcfa,Wang2013-mcfa} and etc.

\appendix
\section{The ECME algorithm}\label{sec:app.ecme}
Recall that $O$ denotes the observed indexes of $\bx$. $\bmu^o$ and $\bmu^m$ denote the corresponding subvectors of $\bmu$. $\bSig^{oo}$, $\bSig^{mo}={\bSig^{om}}'$ and $\bSig^{mm}$ denote the corresponding submatrices of $\bSig$. By the well known normal result, 
\begin{eqnarray*}
	\bx^o&\sim&\cN(\bmu^o,\bSig^{oo}),\label{eqn:xo.dist}\quad \hbox{and}\\
	\bx^m|\bx^o&\sim&\cN(\bmu^{m\cdot o},\bSig^{mm\cdot o}),\label{eqn:xm|o.dist}
\end{eqnarray*}
where
\begin{eqnarray*}
	\bmu^{m\cdot o}&=&\bbE(\bx^m|\bx^o)=\bmu^m-\bSig^{mo}{\bSig^{oo}}^{-1}(\bx^o-\bmu^o),\label{eqn:mu.m|o}\quad \hbox{and}\\
	\bSig^{mm\cdot o}&=&\cov(\bx^m|\bx^o)=\bSig^{mm}-\bSig^{mo}{\bSig^{oo}}^{-1}\bSig^{om}.\label{eqn:Sig.m|o}
\end{eqnarray*}
The observed log likelihood of $\btheta=(\bmu, \bA, \bPsi)$ for incomplete data $\bX_{obs}=\{\bx_n^o\}_{n=1}^N$ is given by \refe{eqn:mislike}. The ECME algorithm that maximizes $\cL_o$ in \refe{eqn:mislike} consists of the following three CM-steps.


\begin{itemize}
	\item \textbf{CML-Step 1:} Given $(\bA,\bPsi)$, maximizing $\cL_o$ in \refe{eqn:mislike} w.r.t. $\bmu$ yields $\tbmu$.
\end{itemize}
Let $\bW_n$ and $\bT_n$ stand for $d\times d$ matrices of zeros except that
$\bW_n^{oo}$ equals to ${\bSig_n^{oo}}^{-1}$ and $\bT_n^{mm}$ equals to $\bSig_n^{mm\cdot o}$. Denote $\hat{\bx}_n=(\hat{\bx}_n^o,\hat{\bx}_n^o)$, where $\hat{\bx}_n^o=\bx_n^o$ and $\hat{\bx}_n^m=\bmu_n^{m\cdot o}$. Following \citep{chuanhai_mlfa}, CML-Step 1 yields
\begin{equation*}
	\tbmu=\left(\sum\nolimits^N_{n=1}\bW_n\right)^{-1}\left(\sum\nolimits^N_{n=1}
	\bW_n\,\bx_n\right) \label{eqn:mis.mean}.
\end{equation*}
\begin{itemize}
	\item \textbf{E-step:}  Given $\bX_{obs}$ and $(\tbmu, \bA, \bPsi$), compute the expected $\cL_1$ to obtain $Q_1$ function.
\end{itemize}
$\bX=(\bX_{obs}, \bX_{mis})$ is the complete data, where $\bX_{mis}=\{\bx_n^m\}_{n=1}^N$. The complete data log likelihood of $\btheta$ for complete data $\bX$ is given by $\cL_1$ in \refe{eqn:comlike1}. The $Q$ function in the E-step is computed as follows
\begin{equation}
	Q_1(\btheta)=\bbE\left(\cL_1(\btheta)|\bX_{obs},
	\tbmu,\bA,\bPsi\right)=-\frac{N}{2}\{\log{|\bSig|}+\mbox{tr}(\bSig^{-1}\bS)\},\label{eqn:mis.Q1}
\end{equation}
where
\begin{equation}
	\bS=\frac1N\sum\nolimits_{n=1}^N\bbE\left( (\bx_n-\tbmu)(\bx_n-\tbmu)'|\bX_{obs},\tbmu,\bA,\bPsi\right)=\frac1N\sum\nolimits_{n=1}^N\left[(\hat{\bx}_n-\tbmu)(\hat{\bx}_n-\tbmu)'+\bT_n\right].\label{eqn:S}
\end{equation}
Let the normalized sample covariance matrix be
\begin{equation}\label{eqn:cov.norm}
	\bar{\bS}=\bPsi^{-1/2}\bS\bPsi^{-1/2},
\end{equation}
where $\bS$ is given in \refe{eqn:S}, and
$(\lambda_i, \mathbf{u}_i)$ be its eigenvalue-eigenvector pairs
of $\bar{\bS}$ sorted in the order $\lambda_1\geq \lambda_2\geq \dots \geq
\lambda_d$. 
\begin{itemize}
	\item \textbf{CMQ-Step 2:} Given $(\tbmu, \bPsi$), maximizing $Q_1$ in \refe{eqn:mis.Q1} w.r.t. $\bA$ yields $\tbA$.
\end{itemize}
Given $\bPsi$, $\tbA$ is obtained by
\begin{equation*}
	\tbA=\bPsi^{1/2}\bU_{k'}\left(\bLmd_{k'}-\bI \right)^{1/2}\bR,
	\label{eqn:A}
\end{equation*}
where, if $\lambda_k>1$, $k'=k$; otherwise, $k'$ is the
unique integer satisfying $\lambda_{k'}>1\geq\lambda_{k'+1}$,
$\bLmd_{k'}=\mbox{diag}(\lambda_1, \lambda_2, \dots,
\lambda_{k'})$, $\bU_{k'}=(\bu_1, \bu_2, \dots, \bu_{k'})$ and $\bR$ is an orthogonal matrix satisfying $\bR\bR'=\bI$.

\begin{itemize}
	\item \textbf{CMQ-Step 3:} Let $
	\bPsi_i\triangleq\mathrm{diag}(\tpsi_1,
	\dots,\tpsi_{i-1},
	\psi_i,\psi_{i+1},\dots,\psi_d)$. Given $(\tbmu, \tbA, \bPsi_i$), maximizing $Q_1$ w.r.t. $\psi_i$ yields $\tpsi_i$, sequentially for $i=1,2,\dots,d$.
\end{itemize}
By FA model assumption that $\bPsi$ is positive, we can pick an arbitrary very
small number $\eta>0$ and assume $\tpsi_i\geq\eta$. Let $\bar{\bA}=\bPsi^{-1/2}\bA$, $\be_i$ be
the $i$-th column of the $d\times d$ identity matrix,
\begin{equation}\label{eqn:Bi}
	\bB_i=\sum\nolimits_{l=1}^{i-1}\tomega_l\,\be_l\be_l'+
	\bI+\bar{\bA}\bar{\bA}',
\end{equation}
$\bb_l$ be the $l$-th column vector of $\bB_i^{-1}$ and $b_{ll}$
stands for the $ll$-th element of $\bB_i^{-1}$. Then $\tpsi_i$ is obtained by
\begin{equation}
	\tpsi_i=\hbox{max}\left\{\left[b_{ii}^{-2} (\bb_i'\bar{\bS}\bb_i
	- b_{ii})+1\right]\psi_i, \eta\right\},\label{eqn:upsigma}
\end{equation} and the required $\tomega_i$ in \refe{eqn:Bi} is given by
\begin{equation}
	\tomega_i=\tpsi_i/\psi_i-1.\label{eqn:upomega}
\end{equation} By \refe{eqn:upsigma}, $\tpsi_l\geq\eta$ and by
\refe{eqn:upomega}, $\tomega_l>-1$, $l=1,\dots,i-1$, thus $\bB_i$ in \refe{eqn:Bi} is invertible and $\tpsi_i$ in \refe{eqn:upsigma} can always be computed.

\section*{Acknowledgements} This work was supported by the National
Natural Science Foundation of China under Grant 12161089, Grant 11761076, and partly by the Science Foundation of Yunnan under Grant 2019FB002.
\bibliography{journals,lit,jhzhao-pub}
\bibliographystyle{elsarticle-harv} 
\end{document}